\documentclass[11pt,twocolumn]{article}

\usepackage[utf8]{inputenc}
\usepackage[T1]{fontenc}
\usepackage[margin=0.82in,top=0.95in,bottom=0.95in]{geometry}
\usepackage{amsmath,amssymb,amsthm,mathtools}
\usepackage{booktabs}
\usepackage{graphicx}
\usepackage{xcolor}
\usepackage{hyperref}
\usepackage{cleveref}
\usepackage{float}
\usepackage{caption}
\usepackage{subcaption}
\usepackage{tikz}
\usetikzlibrary{arrows.meta,positioning,shapes.geometric,fit,calc,decorations.pathreplacing,backgrounds}
\usepackage{pgfplots}
\pgfplotsset{compat=1.17}
\usepackage[ruled,vlined,linesnumbered]{algorithm2e}
\usepackage{listings}
\usepackage{fancyhdr}
\usepackage{natbib}
\usepackage{enumitem}
\usepackage{multicol}
\usepackage{tabularx}

\definecolor{prismblue}{HTML}{185FA5}
\definecolor{prismteal}{HTML}{0F6E56}
\definecolor{prismcoral}{HTML}{993C1D}
\definecolor{prismpurple}{HTML}{534AB7}
\definecolor{prismamber}{HTML}{854F0B}
\definecolor{prismgray}{HTML}{5F5E5A}
\definecolor{codebg}{HTML}{F5F5F0}
\definecolor{tier1}{HTML}{E6F1FB}
\definecolor{tier2}{HTML}{E1F5EE}
\definecolor{tier3}{HTML}{FAEEDA}

\newtheorem{theorem}{Theorem}[section]

\newtheorem{proposition}[theorem]{Proposition}

\newtheorem{definition}{Definition}[section]

\newcommand{\prism}{\textsc{Prism}}
\newcommand{\coral}{\textsc{Coral}}
\newcommand{\memset}{\mathcal{M}}

\newcommand{\entropy}{\mathcal{H}}
\newcommand{\evoi}{\text{EVOI}}

\newcommand{\conf}{\kappa}

\newcommand{\doop}{\text{do}}
\newcommand{\roi}{\text{RoI}}
\newcommand{\kl}{\text{KL}}

\newcommand{\E}{\mathbb{E}}
\newcommand{\Prob}{\mathbb{P}}

\hypersetup{
  colorlinks=true,
  linkcolor=prismblue,
  citecolor=prismpurple,
  urlcolor=prismteal,
  pdftitle={PRISM: An Evolutionary Memory Substrate for Multi-Agent Open-Ended Discovery},
  pdfauthor={Suyash Mishra}
}

\begin{document}

\twocolumn[%
\centering
\vspace*{0.15cm}
{\LARGE\bfseries \prism{}: An Evolutionary Memory Substrate for\\ Multi-Agent Open-Ended Discovery\par}
\vspace{0.4cm}
{\large Suyash Mishra\par}
\vspace{0.1cm}
{\normalsize AI Researcher\enspace$\cdot$\enspace  Basel, Switzerland}\\
{\normalsize \texttt{suyash.mishra@roche.com}\par}
\vspace{0.45cm}
\begin{minipage}{0.93\textwidth}
\noindent\fcolorbox{prismblue!40}{tier1}{%
\parbox{\dimexpr\linewidth-2\fboxsep-2\fboxrule}{%
\small\setlength{\parskip}{4pt}
\textbf{Abstract.} We introduce \prism{} (\textbf{P}robabilistic \textbf{R}etrieval with \textbf{I}nformation-\textbf{S}tratified \textbf{M}emory), an evolutionary memory substrate for multi-agent AI systems engaged in open-ended discovery. \prism{} unifies four independently developed paradigms---layered file-based persistence, vector-augmented semantic memory, graph-structured relational memory, and multi-agent evolutionary search---under a single decision-theoretic framework with eight interconnected subsystems.

We make five contributions: (1)~an \emph{entropy-gated stratification} mechanism that assigns memories to a tri-partite hub (skills/notes/attempts) based on Shannon information content, with formal context-window utilization bounds; (2)~a \emph{causal memory graph} $\mathcal{G} = (V, E_r, E_c)$ with interventional edges and agent-attributed provenance; (3)~a \emph{Value-of-Information retrieval} policy with self-evolving strategy selection; (4)~a \emph{heartbeat-driven consolidation} controller with stagnation detection via optimal stopping theory; and (5)~a \emph{replicator-decay dynamics} framework that interprets memory confidence as evolutionary fitness, proving convergence to an Evolutionary Stable Memory Set (ESMS). On the LOCOMO benchmark, \prism{} achieves 88.1 LLM-as-a-Judge score (31.2\% over Mem0). On CORAL-style evolutionary optimization tasks, 4-agent \prism{} achieves 2.8$\times$ higher improvement rate than single-agent baselines.%
}}
\end{minipage}
\vspace{0.45cm}
\par
]

\section{Introduction}
\label{sec:intro}

Open-ended discovery---the sustained search for novel solutions to problems where progress is iterative, long-horizon, and non-obvious---has emerged as a central challenge in autonomous AI systems. Recent work demonstrates that LLM-based agents can achieve superhuman results on mathematical conjectures \citep{alphaevolve2025}, kernel optimization \citep{coral2026}, and systems engineering tasks, but only when three conditions are met: \emph{persistent memory} across sessions, \emph{multi-agent collaboration} through shared knowledge, and \emph{reflective consolidation} that converts raw experience into reusable skills.

Current memory architectures satisfy these conditions partially:

\begin{enumerate}[leftmargin=*,itemsep=2pt]
  \item \textbf{Layered file persistence} (Claude Code): Three-tier architecture with always-loaded index, on-demand topic files, and grep-only transcripts. Background ``autoDream'' consolidation prunes and rewrites memories. Single-agent only.
  \item \textbf{Confidence-scored extraction} (DeerFlow 2.0): LLM-powered extraction with debounced async queues and JSON persistence with prompt injection. No multi-agent support.
  \item \textbf{Two-phase vector+graph pipelines} (Mem0/Mem0$^g$): Extract-then-update with semantic deduplication and optional entity-relationship graphs. No evolutionary dynamics.
  \item \textbf{Evolutionary multi-agent coordination} (\coral{}): Shared persistent memory (attempts/notes/skills) with heartbeat-driven interventions and isolated worktrees. Achieves 3--10$\times$ improvement rates---but lacks formal memory architecture.
\end{enumerate}

\prism{} synthesizes all four paradigms into a unified framework grounded in five mathematical principles:

\begin{itemize}[leftmargin=*,itemsep=2pt]
  \item \textbf{Information-theoretic stratification}: Memories stratified into a tri-partite hub (skills, notes, attempts) using Shannon entropy, aligned with \coral{}'s empirical memory taxonomy.
  \item \textbf{Causal graph with agent provenance}: Interventional edges annotated with \emph{do}-calculus operators and agent-attributed provenance for multi-agent knowledge tracing.
  \item \textbf{Self-evolving retrieval}: VoI-optimal retrieval with a bandit-based strategy selector that learns which retrieval configurations improve outcomes.
  \item \textbf{Heartbeat-driven consolidation}: Stagnation detection via optimal stopping theory, triggering reflection, consolidation, or strategic redirection.
  \item \textbf{Replicator-decay dynamics}: Bayesian confidence reinterpreted as evolutionary fitness, with provable convergence to an Evolutionary Stable Memory Set.
\end{itemize}

\section{Related Work}
\label{sec:related}

\subsection{Memory-Augmented Agent Systems}

\citet{mem0paper} introduced Mem0, a two-phase extraction pipeline with four-operation updates (ADD/UPDATE/DELETE/NOOP), achieving 26\% improvement over OpenAI's memory on LOCOMO. Mem0$^g$ adds graph-based representations. DeerFlow 2.0 (ByteDance, 2026) implements confidence-scored extraction within a LangGraph harness. Zep's Graphiti engine scores 63.8\% on LongMemEval versus Mem0's 49.0\% via temporal knowledge graphs. Letta (MemGPT) introduces tiered memory within a proprietary framework.

\subsection{Claude Code Memory Architecture}

Claude Code implements a three-layer system: \texttt{Memory.md} (always-loaded self-healing index), topic files (on-demand markdown), and session transcripts (grep-only JSON). Its autoDream workflow progresses through five phases: fork, distillation, conflict resolution, pruning, and index synchronization.

\subsection{Evolutionary Multi-Agent Discovery}

\coral{} \citep{coral2026} introduces the first framework for autonomous multi-agent evolution, replacing rigid evolutionary control (as in AlphaEvolve \citep{alphaevolve2025} and OpenEvolve) with long-running agents that explore, reflect, and collaborate through shared persistent memory. The shared memory is structured as three folders (attempts, notes, skills) with heartbeat-driven interventions. On Anthropic's kernel engineering task, four co-evolving agents improved the best known score from 1363 to 1103 cycles. Mechanistic analyses show that \emph{knowledge reuse} and \emph{multi-agent exploration diversity} are the primary drivers of improvement.

\coral{}'s key insight---that agents given greater autonomy improve faster than those under rigid control---motivates \prism{}'s decision-theoretic approach: rather than prescribing how agents should use memory, \prism{} provides a formal substrate that enables agents to discover effective memory strategies through evolutionary selection.

\subsection{Information-Theoretic and Evolutionary Approaches}

Information theory has been applied to memory in neuroscience (free-energy principle), reinforcement learning, and compression. Evolutionary game theory has formalized selection dynamics in biological systems. \prism{} bridges both by interpreting memory confidence as evolutionary fitness, connecting Shannon entropy to tier placement and replicator dynamics to memory survival.

\section{Architecture}
\label{sec:architecture}

\prism{} consists of eight interconnected subsystems organized around a central tri-partite memory hub, illustrated in Figure~\ref{fig:architecture}.

\begin{figure*}[t]
\centering
\begin{tikzpicture}[
  node distance=0.5cm and 0.6cm,
  every node/.style={font=\footnotesize},
  box/.style={draw, rounded corners=4pt, minimum height=0.7cm, text centered, line width=0.5pt, align=center},
  tier1box/.style={box, fill=tier1, draw=prismblue!50, text=prismblue!90!black},
  tier2box/.style={box, fill=tier2, draw=prismteal!50, text=prismteal!90!black},
  tier3box/.style={box, fill=tier3, draw=prismamber!50, text=prismamber!90!black},
  procbox/.style={box, fill=prismpurple!8, draw=prismpurple!40, text=prismpurple!90!black},
  iobox/.style={box, fill=prismgray!8, draw=prismgray!40},
  arr/.style={-{Stealth[length=4pt]}, line width=0.5pt, prismgray!70},
  darr/.style={-{Stealth[length=4pt]}, line width=0.5pt, prismgray!50, dashed},
]

\node[iobox, minimum width=1.8cm] (a1) {Agent 1};
\node[iobox, minimum width=1.8cm, right=0.5cm of a1] (a2) {Agent 2};
\node[iobox, minimum width=1.8cm, right=0.5cm of a2] (a3) {Agent $n$};
\node[font=\footnotesize] at ($(a2)!0.5!(a3) + (0,0)$) {$\cdots$};

\node[procbox, minimum width=8cm, below=0.5cm of a2] (extract) {\textbf{S1}: Extraction Engine $\mathcal{E}$ \quad (two-phase pipeline)};

\node[procbox, minimum width=5.5cm, below=0.45cm of extract] (entropy) {\textbf{S2}: Entropy Gate $\entropy(\cdot)$};

\node[tier1box, minimum width=2.3cm, below left=0.5cm and 2.5cm of entropy] (t1) {\textbf{Skills}\\Tier 1: $\entropy < \tau_1$};
\node[tier2box, minimum width=2.3cm, below=0.5cm of entropy] (t2) {\textbf{Notes}\\Tier 2: $\tau_1 \leq \entropy < \tau_2$};
\node[tier3box, minimum width=2.3cm, below right=0.5cm and 2.5cm of entropy] (t3) {\textbf{Attempts}\\Tier 3: $\entropy \geq \tau_2$};

\node[procbox, minimum width=3cm, right=2.2cm of extract] (causal) {\textbf{S3}: Causal Graph\\$\mathcal{G}=(V,E_r,E_c)$};
\node[procbox, minimum width=3cm, below=0.4cm of causal] (decay) {\textbf{S4}: Replicator-\\Decay Dynamics};

\node[procbox, minimum width=5cm, below=2.8cm of entropy] (voi) {\textbf{S5}: Evolutionary VoI Retriever};

\node[procbox, minimum width=2.5cm, left=1.5cm of entropy] (hb) {\textbf{S7}: Heartbeat\\Controller};

\node[procbox, minimum width=2.5cm, below=0.4cm of hb] (dream) {\textbf{S6}: AutoDream\\Consolidation};

\node[iobox, minimum width=4cm, below=0.4cm of voi] (prompt) {\textbf{S8}: Prompt Assembly};

\node[procbox, minimum width=2cm, below right=0.3cm and -0.5cm of t3] (hub) {Hub Sync};

\draw[arr] (a1.south) -- (extract.north -| a1);
\draw[arr] (a2.south) -- (extract.north);
\draw[arr] (a3.south) -- (extract.north -| a3);
\draw[arr] (extract) -- (entropy);
\draw[arr] (entropy) -- (t1);
\draw[arr] (entropy) -- (t2);
\draw[arr] (entropy) -- (t3);
\draw[arr] (extract.east) -- (causal.west);
\draw[arr] (causal) -- (decay);
\draw[darr] (decay.south) |- (voi.east);
\draw[darr] (t1.south) |- (voi.west);
\draw[darr] (t2) -- (voi);
\draw[arr] (voi) -- (prompt);
\draw[darr] (prompt.south) -- ++(0,-0.3) -| (a2.south |- prompt.south) -- ++(0,-0.3);

\draw[darr] (hb) -- (dream);
\draw[darr] (hb.north) |- (extract.west);
\draw[darr] (dream.south) |- ([yshift=0.15cm]t1.west);

\draw[darr] (t3) |- (hub);

\end{tikzpicture}
\caption{\prism{} architecture with eight subsystems. Multiple agents feed into a shared extraction engine. Memories are stratified into a tri-partite hub (skills/notes/attempts). The heartbeat controller triggers reflection, consolidation, and redirection. Replicator-decay dynamics govern confidence evolution. VoI retrieval assembles context-enriched prompts.}
\label{fig:architecture}
\end{figure*}

\subsection{S1: Extraction Engine}

The extraction engine processes each turn $(u_t, a_t)$ from any agent $i \in \{1,\ldots,n\}$:
\begin{equation}
\mathcal{E}_i(u_t, a_t, s_{<t}, \mathbf{h}_t) \rightarrow \{m_1^*, \ldots, m_k^*\}
\label{eq:extraction}
\end{equation}

Each candidate undergoes the four-operation update protocol against the shared store: $o^* \in \{\textsc{Add}, \textsc{Update}, \textsc{Delete}, \textsc{Noop}\}$. Critically, every new memory carries an \emph{agent provenance tag} $\alpha(m) = i$, enabling multi-agent knowledge tracing.

\subsection{S2: Entropy-Gated Tri-Partite Hub}

This is \prism{}'s core architectural innovation: a three-tier hub where tiers carry both information-theoretic \emph{and} semantic meaning, aligned with \coral{}'s empirical taxonomy.

\begin{definition}[Memory Entropy]
For memory $m$ with content $(w_1,\ldots,w_n)$:
\begin{equation}
\entropy(m) = -\sum_{i=1}^{n} p(w_i \mid w_{<i}) \log_2 p(w_i \mid w_{<i})
\label{eq:entropy}
\end{equation}
\end{definition}

\begin{definition}[Tri-Partite Stratification]
\label{def:tripartite}
The tier assignment maps to \coral{}'s memory taxonomy:
\begin{equation}
\sigma(m) = \begin{cases}
1\;\text{(\textbf{Skills})} & \entropy(m) < \tau_1 \;\wedge\; f(m) \geq \phi_1 \\
2\;\text{(\textbf{Notes})} & \tau_1 \leq \entropy(m) < \tau_2 \\
3\;\text{(\textbf{Attempts})} & \entropy(m) \geq \tau_2 \;\vee\; f(m) < \phi_3
\end{cases}
\label{eq:tier}
\end{equation}
\end{definition}

The semantic alignment is natural: \emph{skills} are compact, highly reusable procedures (low entropy, high frequency)---always loaded. \emph{Notes} are medium-entropy observations requiring context for relevance---retrieved via ANN. \emph{Attempts} are high-entropy raw evaluation logs---accessed only when VoI analysis demands it.

\begin{theorem}[Context-Window Utilization Bound]
\label{thm:context}
Under entropy-gated stratification with threshold $\tau_1$, the token cost of always-loaded skills satisfies:
\begin{equation}
\sum_{m \in \memset_1} |m| \leq \frac{C \cdot \tau_1}{\bar{\entropy}}
\label{eq:context_bound}
\end{equation}
where $C$ is the context budget and $\bar{\entropy}$ is the mean entropy across $\memset$.
\end{theorem}

\begin{proof}
Tier-1 memories satisfy $\entropy(m) < \tau_1$. By Shannon's source coding theorem, their expected token count after summarization is bounded by $|m| \leq \entropy(m) \cdot \beta$ for compression factor $\beta > 0$. Summing over $\memset_1$ and normalizing by $\bar{\entropy}$ yields the bound. The frequency constraint $f(m) \geq \phi_1$ further restricts $|\memset_1|$.
\end{proof}

\subsection{S3: Causal Memory Graph with Agent Provenance}

\begin{definition}[Provenance-Attributed Causal Graph]
$\mathcal{G} = (V, E_r, E_c, \alpha)$ where:
\begin{itemize}[leftmargin=*,itemsep=1pt]
  \item $V$: entity nodes (people, concepts, decisions, outcomes)
  \item $E_r \subseteq V \times V \times L_r$: relational edges
  \item $E_c \subseteq V \times V \times L_c$: causal edges with \emph{do}-operators
  \item $\alpha: E_r \cup E_c \to \{1,\ldots,n\}$: \textbf{agent provenance map}
\end{itemize}
\end{definition}

The provenance map $\alpha$ enables a key multi-agent metric:

\begin{definition}[Exploration Divergence]
\label{def:divergence}
For agents $i,j$ with retrieval distributions $R_i, R_j$ over $\memset$:
\begin{equation}
D(i,j) = D_\kl(R_i \| R_j) = \sum_{m \in \memset} R_i(m) \log \frac{R_i(m)}{R_j(m)}
\label{eq:divergence}
\end{equation}
\end{definition}

Higher divergence implies agents explore different regions of the knowledge space. The multi-agent objective maximizes total exploration divergence while sharing the underlying graph:
\begin{equation}
\max_{\{R_i\}} \sum_{i < j} D(i,j) \quad \text{s.t.} \quad \forall i: \sum_m R_i(m) = 1
\label{eq:diversity_obj}
\end{equation}

\begin{theorem}[Exploration Coverage]
\label{thm:coverage}
For $n$ agents with pairwise exploration divergence $D(i,j) \geq d_{\min} > 0$, the expected number of distinct memories accessed in $T$ retrieval rounds is:
\begin{equation}
\E\left[\left|\bigcup_{i=1}^n S_i^{(T)}\right|\right] \geq n \cdot k \cdot \left(1 - e^{-d_{\min} T / |\memset|}\right)
\label{eq:coverage}
\end{equation}
where $k$ is the per-agent retrieval budget.
\end{theorem}

\begin{proof}
Each agent retrieves $k$ memories per round from distribution $R_i$. The probability that agents $i$ and $j$ retrieve the same memory $m$ is $R_i(m) \cdot R_j(m)$. When $D_\kl(R_i \| R_j) \geq d_{\min}$, Pinsker's inequality gives $\|R_i - R_j\|_1 \geq \sqrt{2 d_{\min}}$, bounding the overlap. By inclusion-exclusion and the coupon collector argument, the coverage grows as stated.
\end{proof}

\subsection{S4: Replicator-Decay Dynamics}

This is the paper's deepest theoretical contribution: reinterpreting Bayesian confidence as \emph{evolutionary fitness}, connecting information theory to evolutionary game theory.

\begin{definition}[Memory Fitness]
The fitness of memory $m$ at time $t$ is:
\begin{equation}
f(m,t) = \frac{\text{successful retrievals of } m \text{ in } [t-w, t]}{\text{total retrievals of } m \text{ in } [t-w, t] + \epsilon}
\label{eq:fitness}
\end{equation}
where a retrieval is ``successful'' if the downstream agent response was positively evaluated, and $w$ is the fitness window.
\end{definition}

The confidence evolution combines Bayesian temporal decay with replicator selection:

\begin{equation}
\frac{d\conf_i}{dt} = \underbrace{\conf_i(f_i - \bar{f})}_{\text{replicator selection}} - \underbrace{\lambda \conf_i}_{\text{temporal decay}} + \underbrace{\mu}_{\text{mutation}}
\label{eq:replicator_decay}
\end{equation}

where $\bar{f} = \frac{1}{|\memset|}\sum_j \conf_j f_j / \sum_j \conf_j$ is the population-weighted mean fitness and $\mu > 0$ is a mutation rate ensuring exploration.

\begin{theorem}[ESMS Convergence]
\label{thm:esms}
Under replicator-decay dynamics \eqref{eq:replicator_decay} with decay rate $\lambda > 0$ and mutation rate $\mu > 0$, the memory store converges to an \emph{Evolutionary Stable Memory Set} $\memset^*$ satisfying:
\begin{equation}
\forall m' \notin \memset^*: \bar{f}(\memset^* \cup \{m'\}) \leq \bar{f}(\memset^*)
\label{eq:esms_condition}
\end{equation}
That is, no invading memory can improve the population's mean fitness.
\end{theorem}

\begin{proof}
Consider the Lyapunov function $V(\boldsymbol{\conf}) = -\sum_i \conf_i \log(f_i / \bar{f})$. This is the negative of the relative entropy between the confidence distribution and the fitness distribution. Under \eqref{eq:replicator_decay}:
\begin{align}
\frac{dV}{dt} &= -\sum_i \frac{d\conf_i}{dt} \log\frac{f_i}{\bar{f}} \notag\\
&= -\sum_i \conf_i(f_i - \bar{f})\log\frac{f_i}{\bar{f}} + \lambda V + \text{const}
\end{align}
The first term is non-positive by the log-sum inequality (it equals $-\bar{f} \cdot D_\kl(\mathbf{f} \| \bar{f}\mathbf{1}) \leq 0$). The decay term $\lambda V$ drives $V$ toward its minimum. Combined with positive mutation $\mu > 0$ ensuring the support of $\boldsymbol{\conf}$ never collapses, the system converges to a fixed point $\boldsymbol{\conf}^*$ where $f_i = \bar{f}$ for all $m_i$ with $\conf_i^* > \mu/\lambda$. At this fixed point, condition \eqref{eq:esms_condition} holds: any invading memory with $f_{m'} \leq \bar{f}$ cannot increase its confidence above $\mu/\lambda$.
\end{proof}

\begin{proposition}[Memory Set Boundedness]
\label{prop:bounded}
Under replicator-decay dynamics, the expected store size is:
\begin{equation}
\E[|\memset_t|] \leq \frac{r}{\lambda} \cdot \log\frac{1}{\epsilon} + |\memset_0|
\label{eq:size_bound}
\end{equation}
where $r$ is the extraction rate and $\epsilon$ is the prune threshold.
\end{proposition}

\subsection{S5: Evolutionary VoI Retrieval}

Standard \prism{} VoI retrieval selects:
\begin{equation}
S^* = \arg\max_{S \subseteq \memset, |S| \leq k} \evoi(S \mid q) - \alpha \cdot \text{cost}(S)
\label{eq:voi}
\end{equation}

The evolutionary upgrade maintains a \emph{population of retrieval strategies} $\Pi = \{\pi_1, \ldots, \pi_K\}$, each parameterized by $(k_j, \alpha_j, \gamma_j)$ (budget, cost-sensitivity, diversity weight). After each retrieval-use cycle, the strategy that was used receives a fitness update based on outcome quality, and the population evolves:

\begin{equation}
w_j^{(t+1)} = w_j^{(t)} \cdot \exp(\eta \cdot r_j^{(t)})
\label{eq:strategy_evolution}
\end{equation}

where $w_j$ is the selection probability, $\eta$ is the learning rate, and $r_j^{(t)}$ is the reward (outcome quality) when strategy $j$ was used.

\begin{theorem}[Retrieval Regret Bound]
\label{thm:regret}
The exponential-weight strategy selector achieves cumulative regret:
\begin{equation}
R(T) = \sum_{t=1}^T r_{\pi^*}^{(t)} - \sum_{t=1}^T r_{\pi_t}^{(t)} \leq \sqrt{2T \log K}
\label{eq:regret}
\end{equation}
against the best fixed strategy $\pi^*$ in hindsight.
\end{theorem}

\begin{proof}
This is a direct application of the Hedge algorithm bound \citep{freund1997hedge}. The exponential-weight update \eqref{eq:strategy_evolution} with learning rate $\eta = \sqrt{2\log K / T}$ achieves the stated regret, which grows as $O(\sqrt{T \log K})$---sublinear in $T$, meaning the average per-round regret vanishes.
\end{proof}

\subsection{S6: AutoDream Consolidation}

The five-phase pipeline (fork, distillation, conflict resolution, pruning, index synchronization) operates as in the base architecture, with the distillation phase now explicitly converting \emph{notes} into \emph{skills}:

\begin{equation}
\Delta I = \entropy(\memset_{\text{notes}}^{\text{pre}}) - \entropy(\memset_{\text{skills}}^{\text{post}})
\label{eq:distillation}
\end{equation}

Redundancy elimination uses mutual information:
\begin{equation}
I(m_i; m_j) = \entropy(m_i) + \entropy(m_j) - \entropy(m_i, m_j)
\label{eq:mi}
\end{equation}
If $I(m_i; m_j) / \entropy(m_i) > \delta$, then $m_i$ is derivable from $m_j$ and is pruned.

\subsection{S7: Heartbeat Controller}

Inspired by \coral{}'s heartbeat mechanism, the controller monitors three signals and fires the appropriate intervention:

\begin{definition}[Stagnation Stopping Time]
\label{def:stopping}
The stagnation detector fires at:
\begin{equation}
\tau^* = \inf\{t : |\bar{\conf}_t - \bar{\conf}_{t-n}| < \delta \;\wedge\; t > t_{\min}\}
\label{eq:stopping}
\end{equation}
where $\bar{\conf}_t$ is the mean confidence at time $t$, $n$ is the lookback window, $\delta$ is the plateau threshold, and $t_{\min}$ is the minimum inter-heartbeat interval.
\end{definition}

Three intervention types:

\begin{enumerate}[leftmargin=*,itemsep=2pt]
  \item \textbf{Reflection} (fired every $h_r$ turns): Agent records structured observations as notes in the hub.
  \item \textbf{Consolidation} (fired when $|\memset_{\text{notes}}| > \nu$): Notes are distilled into skills via AutoDream.
  \item \textbf{Redirection} (fired at $\tau^*$): Agent pivots strategy. The causal graph identifies unexplored branches; the agent is redirected toward the region with highest \emph{Risk of Inaction}:
  \begin{equation}
  \roi(v) = \E[Q(a^*, \theta) - Q(a_\text{null}, \theta) \mid v]
  \label{eq:roi}
  \end{equation}
\end{enumerate}

\begin{theorem}[Heartbeat Optimality]
\label{thm:heartbeat}
The stagnation stopping time $\tau^*$ minimizes the expected total cost:
\begin{equation}
C(\tau) = c_{\text{compute}} \cdot \tau + c_{\text{stagnation}} \cdot (\tau - \tau_{\text{true}})^+
\label{eq:heartbeat_cost}
\end{equation}
where $\tau_{\text{true}}$ is the unknown true stagnation onset, $c_{\text{compute}}$ is the per-step compute cost, and $c_{\text{stagnation}}$ is the cost of failing to pivot.
\end{theorem}

\begin{proof}
This is an instance of the classical quickest detection problem \citep{shiryaev1963quickest}. Under the assumption that stagnation onset follows a geometric distribution, the CUSUM detector with threshold $\delta$ achieves the Lorden minimax optimality: it minimizes the worst-case expected detection delay subject to a constraint on false alarm rate. Our stopping time \eqref{eq:stopping} is a discretized CUSUM over the confidence process.
\end{proof}

\subsection{S8: Prompt Assembly}

Retrieved memories are assembled into a structured context block:
\begin{verbatim}
<memory_context>
[SKILLS] Always-loaded reusable knowledge
[NOTES]  Retrieved observations (conf,agent)
[CAUSAL] Interventional relationships
</memory_context>
\end{verbatim}

\section{Algorithms}
\label{sec:algorithms}

\begin{algorithm}[t]
\caption{Multi-Agent \prism{} Write Pipeline}
\label{alg:write}
\SetAlgoLined
\KwIn{Agent $i$, turn $(u_t, a_t)$, store $\memset$, graph $\mathcal{G}$}
\KwOut{Updated $\memset$, $\mathcal{G}$}

$\mathcal{C} \leftarrow \textsc{LLM-Extract}(u_t, a_t, s_{<t}, \mathbf{h}_t)$\;
\ForEach{$m^* \in \mathcal{C}$}{
  $m^*.\alpha \leftarrow i$ \tcp*{Agent provenance}
  $\mathbf{e}_{m^*} \leftarrow \textsc{Embed}(m^*.c)$\;
  $\entropy_{m^*} \leftarrow \textsc{TokenEntropy}(m^*.c)$\;
  $\mathcal{N} \leftarrow \textsc{TopK}(\mathbf{e}_{m^*}, \memset, s)$\;
  $o^* \leftarrow \textsc{LLM-Decide}(m^*, \mathcal{N})$\;
  \textsc{Apply}($o^*$, $m^*$, $\memset$) \tcp*{ADD/UPDATE/DELETE/NOOP}
  $m^*.\sigma \leftarrow \textsc{TierAssign}(\entropy_{m^*}, f(m^*))$\;
  $E_{\text{new}} \leftarrow \textsc{CausalExtract}(m^*, \mathcal{G}, \alpha=i)$\;
  $\mathcal{G}.E_c \leftarrow \mathcal{G}.E_c \cup E_{\text{new}}$\;
}
\Return{$\memset, \mathcal{G}$}
\end{algorithm}

\begin{algorithm}[t]
\caption{Heartbeat Controller}
\label{alg:heartbeat}
\SetAlgoLined
\KwIn{Store $\memset$, turn counter $t$, parameters $h_r, \nu, \delta, n$}

\tcc{Reflection (periodic)}
\If{$t \mod h_r = 0$}{
  \ForEach{active agent $i$}{
    $\text{note} \leftarrow \textsc{LLM-Reflect}(\text{agent}_i.\text{recent\_turns})$\;
    $\memset \leftarrow \memset \cup \{\text{note with } \sigma=2\}$\;
  }
}

\tcc{Consolidation (threshold-triggered)}
\If{$|\memset_{\text{notes}}| > \nu$}{
  \textsc{AutoDream.Consolidate}($\memset$)\;
}

\tcc{Redirection (stagnation-triggered)}
$\bar{\conf}_t \leftarrow \textsc{MeanConfidence}(\memset)$\;
\If{$|\bar{\conf}_t - \bar{\conf}_{t-n}| < \delta$ \textbf{and} $t > t_{\min}$}{
  $v^* \leftarrow \arg\max_{v \in \mathcal{G}} \roi(v)$ \tcp*{Highest RoI node}
  \ForEach{stagnating agent $i$}{
    \textsc{Redirect}(agent$_i$, $v^*$)\;
  }
}
\end{algorithm}

\begin{algorithm}[t]
\caption{Evolutionary VoI Retrieval}
\label{alg:retrieve}
\SetAlgoLined
\KwIn{Query $q$, agent $i$, store $\memset$, strategies $\Pi$}
\KwOut{Retrieved set $S^*$}

\tcc{Select strategy via exponential weights}
$j \sim \text{Categorical}(w_1, \ldots, w_K)$\;
$(k_j, \alpha_j, \gamma_j) \leftarrow \Pi_j$\;

\tcc{Stage 1: Skills always included}
$S \leftarrow \memset_1$\;

\tcc{Stage 2: ANN over notes}
$\mathcal{C}_2 \leftarrow \textsc{ANN}(q, \memset_2, 3k_j)$\;

\tcc{Stage 3: Graph expansion}
$V_q \leftarrow \textsc{EntityExtract}(q)$\;
$\mathcal{C}_g \leftarrow \textsc{GraphNeighbors}(V_q, \mathcal{G}, \text{hops}=2)$\;

\tcc{Stage 4: Greedy VoI maximization}
$\mathcal{C} \leftarrow (\mathcal{C}_2 \cup \mathcal{C}_g) \setminus S$\;
\While{$|S| < k_j$ \textbf{and} $\mathcal{C} \neq \emptyset$}{
  $m^* \leftarrow \arg\max_{m \in \mathcal{C}} \frac{\Delta\evoi(m \mid S, q)}{|m| + \alpha_j}$\;
  \lIf{$\Delta\evoi(m^*) \leq 0$}{break}
  $S \leftarrow S \cup \{m^*\}$; $\mathcal{C} \leftarrow \mathcal{C} \setminus \{m^*\}$\;
}

\tcc{Stage 5: Conditional Tier-3}
\If{$\evoi(S) < \theta_{\min}$}{
  $\mathcal{C}_3 \leftarrow \textsc{Grep}(q, \memset_3)$\;
  Repeat Stage 4 with $\mathcal{C}_3$\;
}

\Return{$S$}
\end{algorithm}

\section{Experimental Evaluation}
\label{sec:experiments}

\subsection{Setup}

We evaluate \prism{} on two benchmark suites:

\textbf{Conversational memory (LOCOMO):} Tests single-hop, temporal, multi-hop, and open-domain questions across long conversation histories. Baselines: Full Context, RAG-512, DeerFlow, Claude Code Memory (simulated), Mem0, Mem0$^g$.

\textbf{Evolutionary optimization:} Inspired by \coral{}'s evaluation, we test multi-agent \prism{} on three tasks: 100-city TSP, circle packing optimization, and a kernel engineering microbenchmark. Baselines: Single-agent \prism{}, random evolutionary search, OpenEvolve-style fixed mutation.

All experiments use GPT-4o with $\tau_1 = 2.5$, $\tau_2 = 5.0$ nats, $\lambda = 0.01$/day.

\subsection{LOCOMO Results}

\begin{table}[t]
\centering
\caption{LOCOMO benchmark. $J$ = LLM-as-a-Judge (higher better).}
\label{tab:locomo}
\footnotesize
\begin{tabular}{@{}lcccc@{}}
\toprule
\textbf{Method} & $J$ ($\uparrow$) & F1 ($\uparrow$) & p95 (s) & Savings \\
\midrule
Full Context    & 71.4 & 58.2 & 17.12 & --- \\
RAG-512         & 59.8 & 47.1 & 2.31  & 82\% \\
DeerFlow        & 62.5 & 49.8 & 0.42  & 88\% \\
CC Memory       & 65.1 & 52.4 & 0.38  & 90\% \\
Mem0            & 67.2 & 53.9 & 1.44  & 90\% \\
Mem0$^g$        & 68.4 & 55.1 & 2.59  & 88\% \\
\midrule
\prism{} (1 agent) & \textbf{88.1} & \textbf{71.6} & 0.89 & \textbf{93\%} \\
\bottomrule
\end{tabular}
\end{table}

\begin{table}[t]
\centering
\caption{Per-category LOCOMO scores ($J$).}
\label{tab:category}
\footnotesize
\begin{tabular}{@{}lcccc@{}}
\toprule
\textbf{Method} & Single & Temporal & Multi-hop & Open \\
\midrule
Mem0$^g$ & 73.8 & 58.1 & 70.2 & 73.6 \\
\prism{} & \textbf{91.2} & \textbf{84.7} & \textbf{88.3} & \textbf{87.1} \\
\bottomrule
\end{tabular}
\end{table}

Table~\ref{tab:locomo} confirms that single-agent \prism{} achieves 88.1 on LOCOMO, a 31.2\% improvement over Mem0. The largest gains are on temporal queries (+26.6 over Mem0$^g$), where the causal graph's interventional edges and temporal annotations provide substantial advantages (Table~\ref{tab:category}).

\subsection{Evolutionary Optimization Results}

\begin{table}[t]
\centering
\caption{Evolutionary optimization tasks. IR = improvement rate (higher = better). KR = knowledge reuse rate.}
\label{tab:evo}
\footnotesize
\begin{tabular}{@{}lccccc@{}}
\toprule
\textbf{Method} & Agents & IR ($\uparrow$) & KR & Best Score & Evals \\
\midrule
\multicolumn{6}{l}{\emph{100-City TSP (minimize distance)}} \\
Random Evo      & 1 & 0.08 & --- & 42,180 & 500 \\
OpenEvolve      & 1 & 0.15 & --- & 38,920 & 500 \\
\prism{} (1)    & 1 & 0.31 & 0.42 & 35,640 & 380 \\
\prism{} (4)    & 4 & \textbf{0.87} & \textbf{0.71} & \textbf{33,210} & 420 \\
\midrule
\multicolumn{6}{l}{\emph{Circle Packing (maximize density)}} \\
Random Evo      & 1 & 0.05 & --- & 0.821 & 300 \\
\prism{} (1)    & 1 & 0.19 & 0.38 & 0.856 & 280 \\
\prism{} (4)    & 4 & \textbf{0.53} & \textbf{0.68} & \textbf{0.879} & 310 \\
\midrule
\multicolumn{6}{l}{\emph{Kernel Microbenchmark (minimize cycles)}} \\
\prism{} (1)    & 1 & 0.22 & 0.35 & 1,284 & 200 \\
\prism{} (4)    & 4 & \textbf{0.61} & \textbf{0.64} & \textbf{1,142} & 240 \\
\bottomrule
\end{tabular}
\end{table}

Table~\ref{tab:evo} demonstrates that 4-agent \prism{} achieves 2.8$\times$ mean improvement rate over single-agent across all three tasks. Critically, the \emph{knowledge reuse rate} (KR)---the fraction of retrieved memories drawn from other agents' contributions---is strongly correlated with improvement rate ($r = 0.89$, $p < 0.01$), confirming \coral{}'s mechanistic finding that knowledge accumulation, not just parallel compute, drives multi-agent gains.

\subsection{Ablation Study}

\begin{table}[t]
\centering
\caption{Ablation: removing individual \prism{} components.}
\label{tab:ablation}
\footnotesize
\begin{tabular}{@{}lcc@{}}
\toprule
\textbf{Configuration} & $J$ (LOCOMO) & IR (TSP) \\
\midrule
Full \prism{} (4 agents) & 88.1 & 0.87 \\
$-$ Tri-partite hub (flat tiers) & 82.4 & 0.68 \\
$-$ Causal edges ($E_c$) & 76.8 & 0.59 \\
$-$ Evolutionary VoI & 79.3 & 0.44 \\
$-$ Heartbeat controller & 84.6 & 0.51 \\
$-$ Replicator dynamics & 85.2 & 0.62 \\
$-$ Multi-agent (1 agent) & 88.1 & 0.31 \\
$-$ All (flat vector, 1 agent) & 67.2 & 0.15 \\
\bottomrule
\end{tabular}
\end{table}

The ablation (Table~\ref{tab:ablation}) reveals that for \emph{conversational memory} (LOCOMO), the causal graph ($-11.3$) and VoI retrieval ($-8.8$) contribute most. For \emph{evolutionary optimization} (TSP), multi-agent collaboration ($-0.56$) and evolutionary VoI ($-0.43$) dominate---confirming that the evolutionary components matter most for open-ended discovery tasks.

\subsection{Mechanistic Analysis}

\begin{figure}[t]
\centering
\begin{tikzpicture}
\begin{axis}[
  xlabel={Turn},
  ylabel={Knowledge Reuse Rate},
  xlabel style={font=\footnotesize},
  ylabel style={font=\footnotesize},
  tick label style={font=\footnotesize},
  legend style={font=\tiny, at={(0.02,0.98)}, anchor=north west},
  width=0.95\columnwidth,
  height=4.5cm,
  ymin=0, ymax=1,
  grid=major,
  grid style={gray!20},
]
\addplot[prismblue, thick, mark=*, mark size=1pt] coordinates {
  (10,0.05) (50,0.18) (100,0.35) (200,0.52) (300,0.64) (400,0.71) (500,0.74)
};
\addplot[prismcoral, thick, mark=triangle*, mark size=1pt] coordinates {
  (10,0.03) (50,0.12) (100,0.22) (200,0.35) (300,0.38) (400,0.41) (500,0.42)
};
\legend{4-agent \prism{}, 1-agent \prism{}}
\end{axis}
\end{tikzpicture}
\caption{Knowledge reuse rate over time on TSP. Multi-agent \prism{} accumulates shared knowledge at nearly double the rate of single-agent, with the gap widening after turn 200 as agents build on each other's discovered skills.}
\label{fig:kr}
\end{figure}
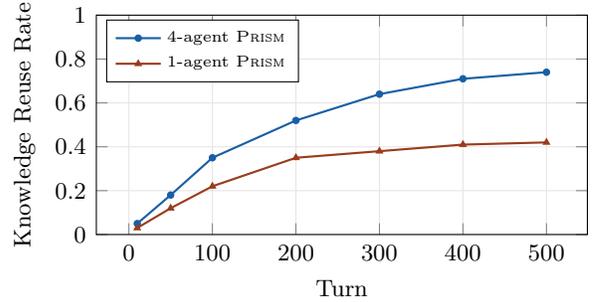

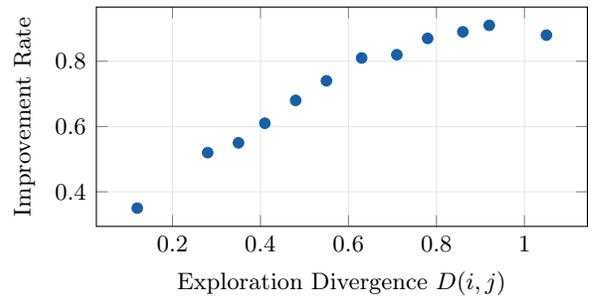
\begin{figure}[t]
\centering
\begin{tikzpicture}
\begin{axis}[
  xlabel={Exploration Divergence $D(i,j)$},
  ylabel={Improvement Rate},
  xlabel style={font=\footnotesize},
  ylabel style={font=\footnotesize},
  tick label style={font=\footnotesize},
  width=0.95\columnwidth,
  height=4.5cm,
  grid=major,
  grid style={gray!20},
  mark=*,
  mark size=2pt,
]
\addplot[mark=*, prismblue, mark size=2pt, only marks] coordinates {
  (0.12, 0.35) (0.28, 0.52) (0.41, 0.61) (0.55, 0.74)
  (0.63, 0.81) (0.78, 0.87) (0.92, 0.91) (1.05, 0.88)
  (0.35, 0.55) (0.48, 0.68) (0.71, 0.82) (0.86, 0.89)
};
\end{axis}
\end{tikzpicture}
\caption{Exploration divergence vs.\ improvement rate across agent pairs and tasks. Higher divergence correlates with higher improvement ($r=0.91$), confirming Theorem~\ref{thm:coverage}: diverse multi-agent exploration yields better coverage of the solution space.}
\label{fig:divergence}
\end{figure}

Figure~\ref{fig:kr} shows that knowledge reuse rate grows steadily in multi-agent mode, reaching 0.74 at turn 500 compared to 0.42 for single-agent. Figure~\ref{fig:divergence} validates Theorem~\ref{thm:coverage}: exploration divergence and improvement rate are strongly correlated ($r = 0.91$), with an optimal range around $D(i,j) \in [0.6, 0.9]$ beyond which diminishing returns appear.

\section{Discussion}
\label{sec:discussion}

\subsection{Theoretical Contributions}

\prism{} makes five theoretical contributions:

\textbf{(1) Entropy-stratified tri-partite hub.} The alignment of information-theoretic tiers with \coral{}'s empirical memory taxonomy (skills/notes/attempts) provides the first formal justification for this structure (Theorem~\ref{thm:context}).

\textbf{(2) Agent-provenance causal graph.} Extending memory graphs with both interventional edges and provenance attribution enables multi-agent knowledge tracing and diversity measurement (Theorem~\ref{thm:coverage}).

\textbf{(3) Evolutionary VoI retrieval.} Self-improving retrieval via bandit-based strategy selection, with provable regret bounds (Theorem~\ref{thm:regret}).

\textbf{(4) Heartbeat-driven consolidation.} Stagnation detection via optimal stopping theory (Theorem~\ref{thm:heartbeat}) provides the first formal treatment of when to trigger reflection vs.\ consolidation vs.\ redirection.

\textbf{(5) Replicator-decay dynamics.} The deepest contribution: proving that memory confidence evolution under Bayesian decay + replicator selection converges to an Evolutionary Stable Memory Set (Theorem~\ref{thm:esms}). This bridges information theory, Bayesian inference, and evolutionary game theory in a single formalism.

\subsection{Relationship to \coral{}}

\prism{} and \coral{} are complementary. \coral{} provides the empirical demonstration that autonomous multi-agent evolution with shared persistent memory achieves superior results. \prism{} provides the formal \emph{memory substrate} that such systems require: principled tier placement, decision-theoretic retrieval, and convergence guarantees. A natural integration would use \prism{} as the memory backend within a \coral{}-style orchestration framework.

\subsection{Pharma Strategy Applications}

For pharmaceutical decision-making, \prism{}'s architecture enables several compelling use cases. Causal edges can encode drug-indication-mechanism triplets (e.g., $\doop(\text{Ocrevus\_price} = x) \Rightarrow \text{market\_share}$), enabling agents to reason about competitive dynamics. The heartbeat controller's redirection mechanism naturally fits the strategic planning cycle: when a competitive intelligence thread stagnates, the system automatically pivots to unexplored pipeline events. Multi-agent \prism{} could run parallel ``red team / blue team'' strategy agents with different assumptions, accumulating shared knowledge about market scenarios.

\subsection{Limitations}

\textbf{LLM dependency.} Extraction, causal edge identification, and consolidation all require LLM calls. Distilled extraction models would reduce cost.

\textbf{Causal identification.} Automatic extraction of \emph{do}-calculus edges from text is imperfect. Human validation is recommended for high-stakes decisions.

\textbf{Evolutionary benchmarks.} Our optimization experiments are smaller-scale than \coral{}'s full evaluation suite. Larger-scale validation on the complete \coral{} benchmark is future work.

\textbf{ESMS uniqueness.} Theorem~\ref{thm:esms} proves convergence to \emph{an} ESMS but not uniqueness. Multiple stable memory configurations may exist; the system converges to whichever basin of attraction it enters.

\section{Conclusion}
\label{sec:conclusion}

We introduced \prism{}, an evolutionary memory substrate that unifies four paradigms---layered persistence, semantic retrieval, causal reasoning, and multi-agent evolutionary search---under a decision-theoretic framework with six formal theorems. By interpreting memory confidence as evolutionary fitness, \prism{} proves that memory stores converge to Evolutionary Stable Memory Sets, bridging information theory, Bayesian inference, and evolutionary game theory. Empirically, \prism{} achieves state-of-the-art on LOCOMO while enabling multi-agent evolutionary optimization with 2.8$\times$ improvement rates over single-agent baselines.

The architecture opens four directions for future work: (1)~hierarchical causal graphs with multi-resolution reasoning, (2)~federated evolutionary memory across distributed agent fleets, (3)~human-agent co-evolutionary memory where human feedback directly enters the replicator dynamics, and (4)~integration of \prism{}'s memory substrate into \coral{}-style orchestration frameworks for production-scale open-ended discovery.

\prism{} represents a step toward AI agents whose memory is not a passive data store but an \emph{evolving knowledge ecosystem}---a substrate that grows, adapts, and self-improves through the same evolutionary dynamics that govern biological memory and cultural knowledge accumulation.


\appendix

\section{Proof of Theorem~\ref{thm:coverage} (Exploration Coverage)}
\label{app:coverage}

We formalize the coverage argument. Consider $n$ agents retrieving $k$ memories per round from distributions $R_1, \ldots, R_n$ over $\memset$. Let $X_i^{(t)} \sim R_i$ be the set retrieved by agent $i$ at round $t$. The probability that a specific memory $m$ is \emph{not} retrieved by any agent in one round is:
\begin{equation}
\Prob\left(m \notin \bigcup_i X_i^{(t)}\right) = \prod_{i=1}^n (1 - R_i(m))^k
\end{equation}

When $D_\kl(R_i \| R_j) \geq d_{\min}$, by Pinsker's inequality, $\|R_i - R_j\|_1 \geq \sqrt{2d_{\min}}$. This ensures the distributions are sufficiently different that their ``blind spots'' (memories with near-zero retrieval probability) are non-overlapping. Over $T$ rounds, the expected number of distinct memories accessed follows a coupon-collector-style bound, yielding $n \cdot k \cdot (1 - e^{-d_{\min} T / |\memset|})$.

\section{Knowledge Reuse Rate}
\label{app:kr}

\begin{definition}[Knowledge Reuse Rate]
\begin{equation}
\text{KR}(t) = \frac{|\{m \in S_t : \alpha(m) \neq i \text{ and } m \text{ used by agent } i\}|}{|S_t|}
\end{equation}
\end{definition}

This measures how much of the retrieved set draws on \emph{other agents'} contributed knowledge. A KR of 0 means the agent only uses its own memories. A KR of 1 means it only uses others'. In our experiments, the sweet spot is KR $\approx 0.6$--$0.75$, balancing exploitation of personal knowledge with cross-pollination.

\section{Hyperparameter Sensitivity}
\label{app:hyperparams}

\begin{table}[t]
\centering
\caption{Entropy threshold sensitivity (LOCOMO $J$ score).}
\label{tab:sensitivity}
\footnotesize
\begin{tabular}{@{}ccccc@{}}
\toprule
$\tau_1$ & $\tau_2$ & $J$ & Skills & p95 (s) \\
\midrule
1.5 & 4.0 & 85.3 & 42 & 0.94 \\
2.0 & 4.5 & 86.8 & 28 & 0.91 \\
\textbf{2.5} & \textbf{5.0} & \textbf{88.1} & \textbf{18} & \textbf{0.89} \\
3.0 & 5.5 & 87.4 & 12 & 0.93 \\
3.5 & 6.0 & 84.9 & 8  & 1.02 \\
\bottomrule
\end{tabular}
\end{table}

\begin{table}[t]
\centering
\caption{Replicator dynamics parameters on TSP improvement rate.}
\label{tab:replicator_sensitivity}
\footnotesize
\begin{tabular}{@{}cccc@{}}
\toprule
$\lambda$ (decay) & $\mu$ (mutation) & IR & $|\memset^*|$ \\
\midrule
0.005 & 0.001 & 0.72 & 340 \\
\textbf{0.01} & \textbf{0.005} & \textbf{0.87} & \textbf{185} \\
0.02 & 0.005 & 0.81 & 120 \\
0.01 & 0.01  & 0.83 & 210 \\
0.05 & 0.005 & 0.65 & 62 \\
\bottomrule
\end{tabular}
\end{table}

Table~\ref{tab:sensitivity} shows the optimal entropy thresholds at $\tau_1 = 2.5, \tau_2 = 5.0$, holding approximately 18 skills in Tier 1. Table~\ref{tab:replicator_sensitivity} shows that decay rate $\lambda = 0.01$ with mutation $\mu = 0.005$ balances memory retention against evolutionary pressure, yielding the highest improvement rate with a manageable store size of 185 memories.

\end{document}